\documentclass[twoside]{article}

\usepackage[accepted]{aistats2022}
\usepackage{amssymb, amsthm}

\usepackage{multirow}
\usepackage{tabularx}

\newcolumntype{Y}{>{\centering\arraybackslash}X}

\theoremstyle{definition}
\newtheorem{example}{Example}[section]

\newtheorem{prop}{Proposition}[section]

\usepackage{mathtools}
\usepackage{bm}
\usepackage[acronym,smallcaps,nowarn,section,nogroupskip,nonumberlist,shortcuts]{glossaries}
\usepackage{url}
\glsdisablehyper
\newacronym{rkhs}{rkhs}{reproducing kernel Hilbert space}
\newacronym{abc}{abc}{approximate Bayesian computation}
\newacronym{mcmc}{mcmc}{Markov chain Monte Carlo}
\newacronym{sigre}{SignatuRE}{ratio estimation with signature kernel logistic regression}
\newacronym{k2re}{K2-RE}{ratio estimation with double kernel logistic regression}
\newacronym{resnet}{ResNet}{residual neural network}
\newacronym{gru}{GRU}{gated recurrent unit}
\newacronym{gru-resnet}{GRU-ResNet}{combination of a GRU model and R\textsc{es}N\textsc{et} classifier}
\newacronym{bresnet}{Bespoke ResNet}{the Bespoke R\textsc{es}N\textsc{et}}
\newacronym{ou}{ou}{Ornstein-Uhlenbeck}
\newacronym{sde}{sde}{stochastic differential equation}
\newacronym{sir}{sir}{sampling importance resampling}
\newacronym{ma2}{ma2}{moving average model of order 2}
\newacronym{gse}{gse}{generalised stochastic epidemic}
\newacronym{swd}{swd}{sliced Wasserstein distances}
\newacronym{wd}{wd}{Wasserstein distances}
\newacronym{lfi}{lfi}{likelihood-free inference}
\newacronym{dre}{dre}{density ratio estimation}
\newacronym{sl}{sl}{synthetic likelihood}
\newacronym{snle}{nle}{neural likelihood estimation}
\newacronym{snpe}{npe}{neural posterior estimation}
\newacronym{mh}{mh}{Metropolis--Hastings}
\newacronym{smcabc}{smc-abc}{sequential Monte Carlo \textsc{abc}}

\newcommand{\bx}{\mathbf{x}}
\newcommand{\by}{\mathbf{y}}

\newcommand{\bs}{\mathbf{s}}
\newcommand{\bw}{\mathbf{w}}

\newcommand{\bK}{\mathbf{K}}
\newcommand{\bU}{\mathbf{U}}
\newcommand{\bD}{\mathbf{D}}

\newcommand{\bth}{\bm{\theta}}
\newcommand{\bTh}{\bm{\Theta}}

\newcommand{\sig}{\mathrm{Sig}}
\newcommand{\sigof}[1]{\mathrm{Sig}(#1)}

\newcommand{\Nys}{Nystr\"{o}m }

\newcommand{\sigre}{S\textsc{ignatu}RE}
\newcommand{\en}{Embedding network}
\newcommand{\pmean}{Posterior mean }
\newcommand{\mi}{Mutual information }

\usepackage[round]{natbib}
\begin{document}

%
\runningtitle{Amortised Inference for Expensive Time-series Simulators with Signatured Ratio Estimation}

%

\twocolumn[

\aistatstitle{Amortised Likelihood-free Inference for Expensive Time-series Simulators with Signatured Ratio Estimation}

\aistatsauthor{ Joel Dyer \And Patrick Cannon \And Sebastian M Schmon }

\aistatsaddress{ 

      University of Oxford\\ 
      {\tt joel.dyer}\\
      {\tt @maths.ox.ac.uk} 

\And  Improbable\\ 
      {\tt patrickcannon}\\
      {\tt @improbable.io} 

\And  Improbable \& Durham University\\ %
      {\tt sebastianschmon}\\
      {\tt @improbable.io} 
} 

]

\begin{abstract}
Simulation models of complex dynamics in the natural and social sciences commonly lack a tractable likelihood function, rendering traditional likelihood-based statistical inference impossible. Recent advances in machine learning have introduced novel algorithms for estimating otherwise intractable likelihood functions using a likelihood ratio trick based on binary classifiers. Consequently, efficient likelihood approximations can be obtained whenever good probabilistic classifiers can be constructed. We propose a kernel classifier for sequential data using \emph{path signatures} based on the recently introduced signature kernel.  We demonstrate that the representative power of signatures yields a highly performant classifier, even in the crucially important case where sample numbers are low. In such scenarios, our approach can outperform sophisticated neural networks for common posterior inference tasks.
\end{abstract}

\section{INTRODUCTION}
\glsresetall

Simulation models are ubiquitous in modern science, arising in fields ranging from the biological sciences \citep[e.g.][]{Christensen2015} to economics \citep[e.g.][]{economics, sbi4abm}. 
Scientific modelling by describing a generative model directly via computer code instead of a probability distribution is appealing as it allows for complex, non-equilibrium mechanics and the exploration of emergent phenomena.

The task of statistical inference for such models is, however, challenging as most simulators lack tractable likelihood functions, precluding the application of traditional likelihood-based inference techniques. 
Enabling \gls{lfi} in arbitrary simulation models has been a fundamental challenge in computational statistics for some time \citep[e.g.][]{diggle1984monte, kennedy2001bayesian}.
A widely used and well researched paradigm is \gls{abc} \citep{pritchard1999population, beaumont2002approximate}, in which the pertinence of parameter values is determined on the basis of the value of a distance between observed data $\by$ and simulation output $\bx$. While appealing, \gls{abc} typically requires many (hundreds of) thousands of calls to the simulator, which is prohibitive for expensive models. 

More recently, neural methods for estimating the likelihood function \citep{papamakarios2019sequential}, posterior density \citep{Greenberg2019}, or likelihood-to-evidence ratio \citep{thomas2016likelihood, hermans2020likelihood}, have been seen to perform competitive likelihood-free inference with far fewer samples \citep{lueckmann2021benchmarking}. 
However, despite the greater sample efficiency provided by these approaches, their budget requirements can still be too high for very complex models, for example high-dimensional spatio-temporal simulations. 
Indeed, a single call to a simulator can take hours to days for multi-scale models of 3D tumour growth \citep[e.g.][]{jagiella2017parallelization} or multiple thousands of CPU hours for climate models \citep[e.g.][]{danabasoglu2020community}.
For others, high simulation budgets may in principle be attainable but undesirable due to the concomitant financial and environmental costs. 
In addition, as recommended by \cite{hermans2020likelihood}, it might be desirable in practice to train multiple density (ratio) estimators to benefit from ensembling and to account for the variance in the density (ratio) estimate.
These considerations give rise to the question of whether (semi-)automatic approaches exist for capturing important features in high-dimensional time-series data \emph{without} the need for large simulation budgets, and make progress in this area important.

In this paper, we analyse a method based on the \emph{signature kernel} \citep{Kiraly2019, salvi2021} for performing \gls{dre} for expensive time-series simulators/low simulation budgets.
It is well known that kernel methods are useful learning tools in low-training-example regimes, providing rich, ready-made data representations \citep{shawe-taylor_cristianini_2004}. 

Moreover, the signature can extract powerful features from time-series data, acting analogously to moment-generating functions for path-valued random variables. 
To benefit from the advantages of both kernels and the signature, we present an approach to \gls{lfi} based on this signature kernel, demonstrating more accurate inferences than competing density ratio techniques when the simulation budget is limited.

\section{BACKGROUND}

In this section, we provide some background on path signatures, sequential kernels as introduced by \cite{Kiraly2019}, and approaches to likelihood-free inference, with a focus on \gls{dre}.

\subsection{Path signatures}\label{sec:signatures}

Let $\mathcal{S}_n(\mathcal{X})$ be the space of length-$n$ time-series on a topological space $\mathcal{X}$ and $\bx \in \mathcal{S}_n(\mathcal{X})$ be a time-series of points $\left(\bx_1, \bx_2, \dots, \bx_n\right)$ observed at times $0 = t_1 < t_2 < \dots < t_n = T$.
Assume we have a (continuous) positive definite kernel $\kappa\colon \mathcal{X} \times \mathcal{X} \to \mathbb{R}$ yielding a \gls{rkhs} $(\mathcal{H}, \kappa)$ through the canonical feature map $x \mapsto \kappa(x, \cdot) \in \mathcal{H}$. 
We consider paths $h \in C([0, T], \mathcal{H})$ with $h(0) = 0\in\mathcal{H}$ and
\begin{equation*}
    \|h\|_1 = \sup_{\pi(0, T)}\sum_{i=1}^{n-1} \|h(t_{i+1}) - h(t_i)\|_\mathcal{H} < \infty,
\end{equation*}
where the supremum is taken over all finite partitions $\pi(0, T)$ of $[0, T]$, and we may construct such a path from $\bx$ by linearly interpolating the $\kappa(\bx_i, \cdot)$. The \textit{signature}, denoted $\sig$, \citep[see e.g.][]{lyons2014rough} then maps such paths into a series of tensors (by convention, $\mathcal{H}^{\otimes 0} = \mathbb{R}$),
\begin{equation}\label{eq:Sig}
    h \mapsto \sigof{h} := \left(1, S_1(h), S_2(h), \dots\right) \in \prod_{m \geq 0} \mathcal{H}^{\otimes m},
\end{equation}
in which the $m$-th degree component $S_m(h)$ consists of the $m$-th moment tensor of the path integral:
\begin{equation*}
    S_m(h) := \int_0^T \mathrm{d}h^{\otimes m} := \int_0^T \int_0^t \mathrm{d}h^{\otimes (m-1)} \otimes \mathrm{d}h(t),
\end{equation*}
with $\int dh^{\otimes 0} = 1$.

\begin{example}[\citet{Kiraly2019}]
\label{example_exp}

Let $h(t)$ take values in $\mathbb{R}^2$, $h(t) = (h_1(t), h_2(t))$. Then
\begin{equation*}
    S_1(h) = 
    \begin{bmatrix}
        \int_0^T dh_1(t) & \int_0^T dh_2(t) 
    \end{bmatrix}'
\end{equation*}
where $'$ is the transpose, and $S_2(h)$ is
\begin{equation*}
    \begin{bmatrix}
        \int_0^T \int_0^{t_2} dh_1(t_1)dh_1(t_2) & \int_0^T \int_0^{t_2} dh_1(t_1)dh_2(t_2) \\[4pt]
        \int_0^T \int_0^{t_2} dh_2(t_1)dh_1(t_2) & \int_0^T \int_0^{t_2} dh_2(t_1)dh_2(t_2)
    \end{bmatrix}.
\end{equation*}
\end{example}

This general approach allows us to lift time-series data into an \gls{rkhs}, which 
can be particularly useful when the underlying data consists of sequences of non-Euclidean data e.g. graphs or images.

Signatures have several additional favourable properties which make them theoretically appealing: they are a continuous map; they uniquely identify paths, in practice\footnote{The signature is injective up to tree-like equivalence \citep{Hambly_2010}, which is easily remedied with time-augmentation $h(t) \mapsto (t, h(t))$ \citep{levin2016learning}.}; and they are \emph{universal non-linearities} \citep[see e.g.][for a proof]{Kiraly2019}. This latter property means that for any compact set $\mathcal{K}$ of paths of bounded variation, any function $f \in C(\mathcal{K}, \mathbb{R})$ can be approximated uniformly by linear functionals of the signature, i.e. for any $\varepsilon>0$ there exists a linear functional $L$
\begin{equation*}
    \sup_{h\in \mathcal{K}}\Big|f(h) - L\left[\sig(h)\right]\Big| < \varepsilon.
\end{equation*}
In particular, this suggests that we can learn classifiers by linearly regressing the logit on the signature. 

Computing iterated integrals over potentially Hilbert space valued paths might seem infeasible for practical applications. However, a kernel trick \citep{Kiraly2019, salvi2021}, described below, allows for efficient computation of inner products. 

\subsection{The signature kernel}

We can kernelise the feature map \eqref{eq:Sig} with the inner product between $A = \left(a_0, a_1, \dots\right)$ and $B = \left(b_0, b_1, \dots\right)$, $A, B \in \prod_{m \geq 0} \mathcal{H}^{\otimes m}$ as
\begin{gather}\label{eq:inner_prod}
    \left<A, B\right> = \sum_{m = 0}^{\infty} \left<a_m, b_m\right>_{\mathcal{H}^{\otimes m}},\ \  \text{where}\\\nonumber
    \left<u_{1} \otimes \dotsb \otimes u_{m}, w_{1} \otimes \dotsb \otimes w_{m}\right>_{\mathcal{H}^{\otimes m}} = \prod_{k = 1}^{m} \left<u_{k}, w_{k}\right>_{\mathcal{H}}.
\end{gather}
The \emph{signature kernel} over $\kappa$ for $\mathcal{X}$-valued paths $x$, $y$,
\begin{equation}\label{eq:sigkern}
    k\left(x, y\right) = \langle{\sig(\kappa(x, \cdot)), \sig(\kappa(y, \cdot))}\rangle
\end{equation}
yields a positive-definite, universal kernel in which the underlying paths are first lifted from $\mathcal{X}$ into paths evolving in a feature space $\mathcal{H}$ via $\kappa$, before entering the inner product \eqref{eq:inner_prod} \citep{Kiraly2019}. Figure \ref{fig:sigkern} shows a schematic illustrating how the different kernels embed the time-series in the signature kernel. \cite{Kiraly2019} further show that \eqref{eq:sigkern} can be efficiently evaluated using a Horner scheme only relying on evaluations of $\kappa$; additionally, \cite{salvi2021} show that the untruncated signature kernel can be estimated by solving a Goursat partial differential equation. Equipped with the signature kernel, we will be able to learn classifiers by linearly regressing the logit on the signature using kernel logistic regression.

\begin{figure}
\centering
\includegraphics[width=\linewidth]{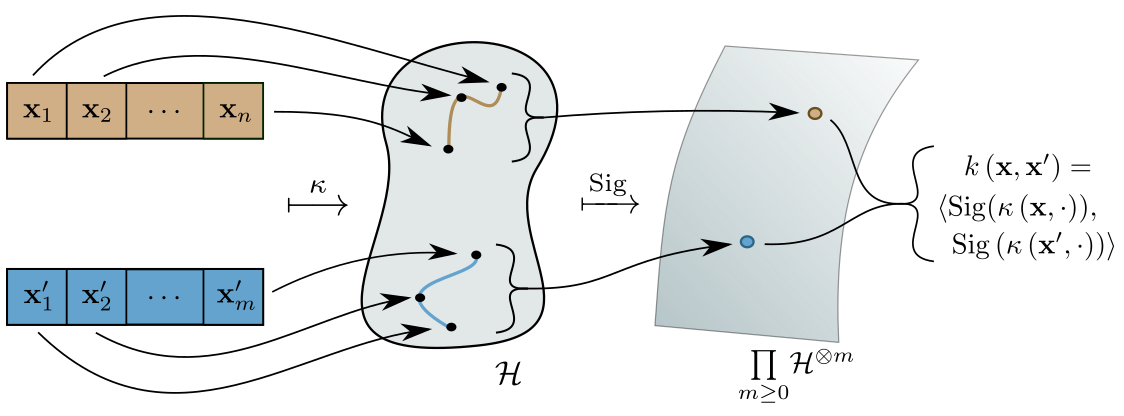}
\vskip -0.05in
\caption{
Time-series embedding via the signature kernel $k$ with 
static kernel $\kappa$. The time-series $\mathbf{x}$, $\tilde{\mathbf{x}}$ are lifted to paths in feature space $\mathcal{H}$, via $\kappa$ and some interpolation scheme, before being mapped to a Hilbert space $\prod_{m\geq 0} \mathcal{H}^{\otimes m}$ of tensors via the signature.}\label{fig:sigkern}
\end{figure}

\subsection{Likelihood-free inference}
\glsreset{lfi}

Many approaches to \gls{lfi} have been proposed. 
Among them, a common theme is approximation of the true likelihood function or posterior density.
\Gls{abc} (see \citeauthor{beaumont2019approximate}, \citeyear{beaumont2019approximate}, for a recent review) is a family of methods in which samples from an approximate posterior are derived through forward simulation of the model, $\bx \sim p(\bx \mid \bth)$, in combination with a summary statistic $\bs$ and distance function $D(\bs(\bx), \bs(\by))$ capturing 
a meaningful discrepancy between simulated and real data.
It can be seen as an instance of kernel density estimation since the induced likelihood approximation permits the expression 
\begin{equation*}
    p(\by\mid \bth) \approx \frac{1}{Q} \sum_{i=1}^Q K_{\varepsilon}\left(D(\bs(\bx^{(i)}), \bs(\by))\right)
\end{equation*}
where $\bx^{(i)} \overset{iid}{\sim} p(\bx \mid \bth)$ and $K_{\varepsilon}$, a kernel function with window $\varepsilon$, largely controls the quality of the approximation.
In contrast, a number of methods for \gls{lfi} entail constructing explicit models of the likelihood function or posterior density. An early example is synthetic likelihood \citep{Wood2010}, in which $p\left(\bs\left(\bx\right) \mid \bth\right)$ is modelled as a multivariate Gaussian with mean and covariance estimated from $Q > 1$ simulations at $\bth$. More recent examples include \gls{snle} \citep{papamakarios2019sequential} and \gls{snpe} \citep{Greenberg2019}, in which $p\left(\bs\left(\bx\right) \mid \bth\right)$ and $p\left(\bth \mid \bs\left(\bx\right)\right)$, respectively, are estimated with highly flexible neural conditional density estimators, in particular normalising flows.

\subsection{Amortised density ratio estimation}\label{sec:ADRE}

We briefly recapitulate \gls{dre} for \gls{lfi}, first introduced by \cite{thomas2016likelihood}, to estimate the likelihood-to-evidence ratio 
$$
    r\left(\bx, \bth\right) = \frac{ p\left(\bx \mid \bth\right)}{ p(\bx)}
$$
and thus the parameter posterior $p\left(\bth \mid \bx\right)$ given a prior distribution $p\left(\bth\right)$. 
Most relevant for us is \emph{amortised} \gls{dre} \citep{hermans2020likelihood}.
The core idea is to train a binary classifier to distinguish between positive examples $\left(\bx, \bth \right) \sim p\left(\bx \mid \bth\right) p\left(\bth\right)$ with label $z=1$ and negative examples $\left(\bx, \bth \right) \sim p\left(\bx\right)p\left(\bth\right)$ with label $z=0$. 
The optimal decision boundary is then
\begin{equation}
    d\left(\bx, \bth\right) = \frac{p\left(\bx, \bth\right)}{p\left(\bx, \bth\right) + p\left(\bx\right)p\left(\bth\right)},
\end{equation}
permitting posterior density evaluations as
\begin{equation}
    p\left(\bth \mid \bx\right) = \frac{d\left(\bx, \bth\right)}{1 - d\left(\bx, \bth\right)} p\left(\bth\right) = r\left(\bx, \bth\right)p\left(\bth\right).
\end{equation}
In practice, only an approximation $\hat{r}\left(\bx, \bth\right)$ is available. 
Such approximations can be used for posterior sampling with \gls{mcmc} \citep[e.g.]{pham2014note, thomas2016likelihood, hermans2020likelihood} 
or to perform likelihood ratio tests for frequentist inference \citep{cranmer2015approximating, pmlr-v119-dalmasso20a}. 
As with neural likelihood and posterior estimation, we say that \gls{dre} -- in the form suggested by  \cite{hermans2020likelihood} -- can be \emph{amortised} since $\hat{r}\left(\by, \bth\right)$ can be evaluated for any observation $\by$ and any parameter $\bth$ without retraining the density estimator. 

\subsection{Summary statistics}

For many \gls{lfi} methods, it is typically necessary to reduce high-dimensional data $\bx$ into summary statistics $\bs\left(\bx\right)$. 
A number of approaches for doing so have been explored for \gls{abc}, including semi-automatic \gls{abc} \citep{fearnhead2012constructing}, in which summary statistics $\bs\left(\by\right) = \mathbb{E}\left[\bth \mid \by\right]$ are estimated by performing a vector-valued regression of $\bth^{(i)}$ onto $g\left(\bx^{(i)}\right)$ for training data $\lbrace{\left(\bx^{(i)}, \bth^{(i)}\right)\rbrace}_{i=1}^{Q} \sim p\left(\bx, \bth\right)$ and candidate summary statistics $g\left(\cdot\right)$, and summary-free automatic methods which compute distances on the full dataset without the need to first compute summary statistics \citep{Park2016, Bernton2019, dyer2021approximate}. More recently, \cite{Chen2020} explored the possibility of learning approximately sufficient, mutual information-maximising summary statistics with neural networks.

Some of these methods remain applicable to \gls{snpe} and \gls{dre}. For example, \cite{DinevGutmann} use a convolutional neural network to learn $\bs\left(\bx\right) = \mathbb{E}\left[\bth \mid \bx\right]$, which are then used as predictors in a logistic regression model for \gls{dre}. \Gls{snpe} and \gls{dre} are however particularly interesting in that they permit the concurrent learning of both summary statistics and densities/ratios by augmenting a classifier with an initial \emph{embedding network} \citep{Lueckmann2017}. 
This composite summary-learning/posterior-estimating network is trained end-to-end on the same loss function, producing competitive results \citep{Greenberg2019}. 
However, learning relevant features/summary statistics from neural networks in scenarios where sampling budgets are prohibitively low can be challenging. 

For later reference, we tabulate some of the key works involving learning summary statistics for time-series data in \gls{lfi} settings. We list for each the adopted training scheme, the number of trainable parameters in each case (each involved neural networks), and the assumed simulation budgets in Table \ref{tab:paramcounts_budgets}. 

\begin{table*}[t]
\vskip -0.1in
\caption{Summary of network sizes and simulation budgets for summary statistic learning in previous works.}
\label{tab:paramcounts_budgets}
\begin{center}
\begin{tabular}{c c c c}
\textbf{Authors} & \textbf{Learning method} & \textbf{Network size} & \textbf{Simulation budget} \\
\hline & {} & {} & {} \\
 \cite{jiang2017learning} & \pmean as summary & $\sim 3\times 10^{4}$ & $10^{6}$ \\
 \cite{Lueckmann2017} & \en & $\sim 2\times 10^{3}$ & $5\times 10^{3} - 2.5\times 10^{4}$ 
 \\
 \cite{DinevGutmann} & \pmean as summary & 8,422 & $10^{5}$ \\
 \cite{Greenberg2019} & \en & $\sim 3\times 10^{4}$ & Between $10^{3}$ and $\sim 10^{4}$ \\
 \cite{Chen2020} & \mi maximisation & $\sim 1.5 \times 10^{4}$ & $10^{3} - 10^{4}$ 
 \\
\cite{dyer2021deep} & \en & $\sim 10^{4}$ & $10^{3} - 10^{4}$ 
\end{tabular}
\end{center}
\vspace{-0.15in}
\end{table*}

\section{METHOD}

Our goal is to perform amortised density ratio estimation as described in Section \ref{sec:ADRE} for expensive simulators/low simulation budgets. As we will see in experiments below, learning both summary statistics and a classifier can be challenging in such regimes. To ameliorate this, we propose to build a classifier that leverages the signature kernel, which defines a universal kernel for multivariate and possibly irregularly sampled sequential data. The core idea is that using the predefined features captured by the signature and made available by the signature kernel may yield a more reliable density ratio estimator in low-sample regimes than alternative methods for which summary statistics must be learned.

To construct a probabilistic binary classifier using the signature, we may use the fact that a third kernel $m$ on $\mathcal{S}_n(\mathcal{X}) \times \bTh$ can be composed given two kernels $k : \mathcal{S}_n(\mathcal{X}) \times \mathcal{S}_n(\mathcal{X})  \to \mathbb{R}$ and $l : \bTh \times \bTh \to \mathbb{R}$ as
\begin{equation}\label{eq:ProductKernel}
    m\left(\left(\bx, \bth\right), (\tilde{\bx}, \tilde{\bth})\right) = k\left(\bx, \tilde{\bx}\right) l(\bth, \tilde{\bth}).
\end{equation}
Taking $k$ to be the signature kernel \eqref{eq:sigkern} and $l$ to be a standard universal kernel on $\bTh$, we may construct a kernel-based binary classifier for the purpose of performing \gls{dre} for expensive time-series simulators, in this sense bypassing the need to learn summary statistics in addition to a density (ratio) estimator.

For a regularisation constant $\omega \in \mathbb{R}_{+}$, training a kernel binary classifier with loss $\ell$ amounts to solving the optimisation problem 
\begin{equation}\label{eq:klr}
    \min_{f \in \mathcal{H}_m} \sum_{i=1}^N \ell\left(f(\bx^{(i)}, \bth^{(i)}), z_i\right) + \frac{\omega}{2} \| f \|^2_{\mathcal{H}_m},
\end{equation}
where $\mathcal{H}_m$ is the \gls{rkhs} associated with $m$ and $z_i$ is the class label associated with data $(\bx^{(i)}, \bth^{(i)})$. 
By the representer theorem, the solution to \eqref{eq:klr} is of the form
\begin{equation}\label{eq:SolKLR}
    f\left(\bx, \bth\right) = \sum_{i=1}^{N} c_i k\left(\bx^{(i)}, \bx\right) l\left(\bth^{(i)}, \bth\right)
\end{equation} 
for real coefficients $c_i$. Throughout, we use the logistic loss as $\ell$, since this is known to yield classifiers with well-calibrated probability estimates.
This approach to learning the likelihood-to-evidence ratio is appealing since $m$ is a universal kernel:
\begin{prop}
Let $\mathcal{H}$ be a Hilbert space, $\mathcal{K}$ a compact set of continuous $\mathcal{H}$-valued paths of bounded variation on $\left[0, T\right]$, and assume that $\forall X\in \mathcal{K}$, $X$ has at least one monotone coordinate and $X(0) = \text{constant}$. Also let $k: \mathcal{K} \times \mathcal{K} \to \mathbb{R}$ be the signature kernel and $l$ be a universal kernel on $\bTh$. Then $m$ as defined in Equation \eqref{eq:ProductKernel} is a universal kernel on $\mathcal{K} \times \bTh$.
\end{prop}

\begin{proof}
From \citet[][Theorem 1]{Kiraly2019}, the signature kernel is a universal kernel on $\mathcal{K}$. Then the assumed universality of $l$ and \citet[][Lemma 5.2]{generalizing2011} give the desired result.
\end{proof}

The universality of $m$ then enables us to learn an estimate of the density ratio arbitrarily well.

\subsection{Low-rank approximation}

Computing the signature kernel for all pairs $\left(\bx^{(i)}, \bx^{(j)}\right)$ in the simulated dataset can be expensive if the $\bx^{(i)}$ are long and/or are high-dimensional. We therefore use the kernel $m$ defined in \eqref{eq:ProductKernel} -- with $k$ the signature kernel and $l : \bTh \times \bTh \to \mathbb{R}$ an anisotropic Gaussian radial basis function (RBF) kernel -- and the \Nys approximation to first find a representation of each pair $\left(\bx, \bth\right)$ before feeding this low-dimensional approximation into the logistic regression model. For a given kernel $m$, the \Nys method \citep{williams2001using, NIPS2012_621bf66d} provides a low-dimensional approximation $\hat{\phi}$ of the high- or potentially infinite-dimensional feature map $\phi(v) := m\left(v, \cdot\right)$ as follows: assume the kernel $m$ is of rank $q$ such that for any data $\lbrace{v^{(i)}\rbrace}_{i=1}^{N}$ we may write the corresponding Gram matrix $\bK$ as
\begin{equation}
    \bK = \bU \bD \bU',
\end{equation}
where $\bU \in \mathbb{R}^{N \times q}$ is the matrix of eigenvectors and $\bD = \mathrm{diag}\left(\lambda_1, \dots, \lambda_q\right) \in \mathbb{R}^{q\times q}$ is the diagonal matrix consisting of eigenvalues $\lambda_i$. Then denoting the first $q$ rows of $\bU$ as $\bU_q$, we may find an approximate feature representation of $v$ under $m$ as \citet{NIPS2012_621bf66d}
\begin{equation*}\label{eq:nys_features}
    \hat{\phi} \left(v\right) =  \bD^{ -\frac{1}{2} } \bU_q' \left[m\left(v, v^{(1)}\right), \dots, m\left(v, v^{(q)}\right)\right]'.
\end{equation*}
Using these approximate feature representations obtained with the \Nys approximation, we then construct a linear logistic regression model by solving the following optimisation problem:
\begin{equation}
    \min_{\bw \in \mathbb{R}^{q}} \sum_{i=1}^N \ell\left(\bw' \hat{\phi}\left(v^{(i)}\right), z_i\right) + \frac{\omega}{2} \| \bw \|^2_{2},
\end{equation}
where $\ell$ is the logistic loss. We omit the use of an intercept in the linear logistic regression optimisation problem above for simplicity, but include it in practice.

Throughout the rest of this paper, we term this approach to performing ratio estimation with the signature kernel and logistic regression \sigre{}.

\section{EXPERIMENTS}\label{sec:Exp}

In this section, we present experiments on the relative performance of the \sigre{} method against possible alternatives for \gls{dre} in likelihood-free inference contexts. 
For each task, we compare the quality of the posterior estimated with \sigre{} against the posteriors estimated with three alternatives:

\begin{enumerate}
    \item a neural network consisting of a gated-recurrent unit (GRU) and residual network (R\textsc{es}N\textsc{et}), jointly termed GRU-R\textsc{es}N\textsc{et}. The GRU has trainable parameters $\varphi$ and consists of two stacked GRU layers of size 32. The GRU and R\textsc{es}N\textsc{et} are trained concurrently on the cross-entropy loss, so that the GRU learns a low-dimensional summary $\bs_{\varphi}\left(\bx\right)$ as the R\textsc{es}N\textsc{et} learns the density ratio;
    \item a R\textsc{es}N\textsc{et} which instead consumes predefined, hand-crafted summary statistics $\tilde{\bs}(\bx)$ that are tailored to the inference task and known to be informative of the parameters to be inferred for tractable simulation models, or that are commonly used elsewhere in the literature when the simulation model is not tractable. Such an approach should be considered a gold standard that is not generally available for complex, opaque simulation models whose structure cannot be exploited to derive suitable summary statistics. We refer to this method as the B\textsc{espoke} R\text{es}N\textsc{et};
    \item \gls{k2re}, a modification of K2-ABC \citep{Park2016} that we propose as an alternative kernel-based method for \gls{dre}. The setup is identical to \sigre{} with the exception that, instead of the signature kernel, we use
    \begin{equation}
        k\left(\bx, \tilde{\bx}\right) = \exp\left(-\frac{\widehat{\textsc{MMD}}^2\left(\mu_{\bx}, \mu_{\tilde{\bx}}\right)}{\epsilon}\right)
    \end{equation}
    as the positive definite kernel on $\bx$, where $\mu_{\bx}$ is the empirical measure consisting of the $n_{\bx}$ points comprising $\bx$ and $\widehat{\text{MMD}}^2\left(\mu_{\bx}, \mu_{\tilde{\bx}}\right)$ 
    is an unbiased estimate of the kernel maximum mean discrepancy between $\mu_{\bx}$ and $\mu_{\tilde{\bx}}$ for an appropriate kernel $\chi$ \citep[see Section 3.][]{Park2016}. We use a Gaussian RBF and the median heuristic \citep[Section 4.][]{Park2016} for $\chi$. We include further details on this method in the supplement.
\end{enumerate}

For the static kernel $\kappa$ in the signature kernel (see Section \ref{sec:signatures}), we use a Gaussian RBF kernel with scale parameter chosen as $\mathrm{median}\{\| \by_i - \by_j  \|^2_{i,j}\}$, where $\by = \left(\by_1, \dots, \by_n\right)$ is the observation\footnote{Tuning this scale parameter may be expensive for, and thus is a limitation of, our particular implementation. However, cheaper implementations exist for the signature kernel (see e.g. \url{https://github.com/tgcsaba/KSig}).}. For the kernel $l : \bTh \times \bTh \to \mathbb{R}$, we use an anisotropic Gaussian RBF kernel. To tune the length scale hyperparameters for $l$, the regularisation parameter $\omega$, and the $\epsilon$ parameter for \gls{k2re}, we use Bayesian optimisation 
and 5-fold cross-validation (see the Supplementary Material for further details). To train the logistic regression models, we use the \textsc{l-bfgs} algorithm \citep{LBFGS} with a maximum number of 500 iterations.

To construct the set of negative examples $\left(\bx, \bth \right) \sim p\left(\bx\right)p\left(\bth\right)$ for \sigre{} and \gls{k2re}, we choose a proportion $K > 0$ of the $\bx^{(i)}$ and pair them with some $\bth^{(j)}$, $j \neq i$. $K > 1$ may also be chosen, in which case some $\bx^{(i)}$ will appear multiple times in the set of negative examples. Unless stated otherwise, we take $K=1$ and $q=B_{\rm min} (K+1)$ in the \Nys approximation for both \sigre{} and \gls{k2re}, where $B_{\rm min}$ is the smallest simulation budget considered in the experiment\footnote{This value for $q$ is chosen since it is the largest value that can be consistently applied across the range of simulation budgets considered in a given experiment.}.

\subsection{Computational expense}

Evaluation of the signature kernel has complexity linear in the dimension of the time-series and linear (resp. quadratic) in the length of the time-series when evaluated on CPU (resp. GPU). 
Empirically, we observe \sigre{} to entail a comparable computational cost to the \textsc{GRU-ResNet}, the former typically requiring 3-5 CPU hours for training and inference and the latter typically requiring 1-2 CPU hours. For the simulation models for which we suppose our approach may be most helpful -- those with significantly limited simulation budgets -- we expect this to amount to a negligible difference: 1-4 additional CPU hours would allow for few or no additional simulations to be generated.

\begin{figure}
\centering
\includegraphics[width=0.42\textwidth, trim=5 10 0 5, clip=True]{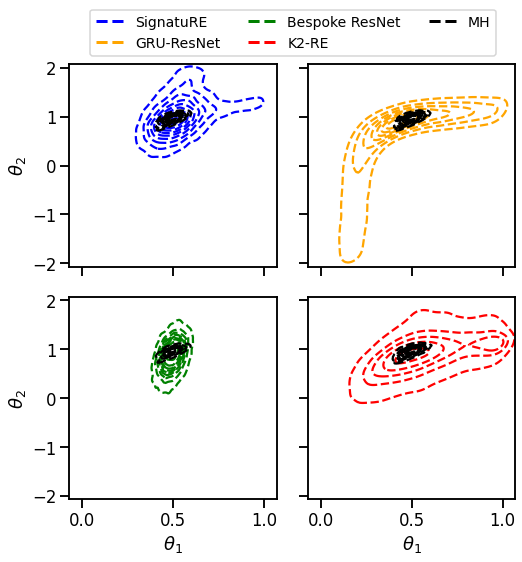}
\caption{(\textbf{Ornstein-Uhlenbeck}) Posteriors obtained with \sigre{} (blue, top left), \gls{gru-resnet} (orange, top right), \gls{bresnet} (green, bottom left), and \gls{k2re} (red, bottom right) for a budget of 500 simulations and the approximate ground truth posterior obtained using the true likelihood function and Metropolis-Hastings (black).}\label{fig:ou_posteriors}
\end{figure}

\subsection{Ornstein-Uhlenbeck process}

The \gls{ou} process \citep{UhlenbeckG.E.1930Otto} is a prototypical Gauss--Markov \gls{sde} model. We discretise the \gls{sde} such that the data $\bx = \left(\bx_0, \bx_1, \dots, \bx_T\right), \bx_i\in \mathbb{R}$ is generated according to
\begin{equation*}
    \bx_{i} = \theta_1 \exp{(\theta_2)} \Delta t + (1 - \theta_1 \Delta t) \bx_{i-1} + \frac{\epsilon_{i}}{2},
\end{equation*}
where $\Delta t = 0.2$ is the time discretisation, $\bth = \left(\theta_1, \theta_2\right)$ are the model parameters to be inferred, $T=50$, and $\epsilon_i \sim \mathcal{N}\left(0, \Delta t\right)$. We generate $\bx^{*} \sim p\left(\mathbf{x} \mid \bth^{*}\right)$ with $\bth^{*} = \left(0.5, 1\right)$ and consider the task of estimating $p\left(\bth \mid \bx^{*}\right)$ given priors $\theta_1 \sim \mathcal{U}\left(0, 1\right)$ and $\theta_2 \sim \mathcal{U}\left(-2,2\right)$.

We compare \sigre{} against the alternative \gls{dre} methods described in Section \ref{sec:Exp}. As $\tilde{\bs}\left(\bx\right)$, we use the intercept and slope of a linear regression of $\bx_t$ vs. $\bx_{t-1}$ as estimated with least squares (i.e. the maximum likelihood estimate) and the mean value of $\bx$. These estimate $\theta_1 \exp{(\theta_2)} \Delta t$, $1 - \theta_1 \Delta t$, and $\exp\left(\theta_2\right)$, respectively, and are thus informative summary statistics for $\bth$. For \gls{gru-resnet}, we apply a linear layer of size 3 after the GRU in order to match the dimension of $\tilde{\bs}$, resulting in a GRU with 9,795 trainable parameters.

In Figure \ref{fig:ou_posteriors} we show contour plots obtained by pooling the samples obtained from each ratio estimation method with a simulation budget of 500 simulations across 20 different seeds. Samples from the approximate ground truth posterior, obtained with \gls{mh} (see Appendix for details) and the true likelihood function, are shown with black contour lines throughout. Additionally, we show in Figure \ref{fig:ou_swds} the \gls{wd} between the estimated posteriors and the approximate ground truth posterior\footnote{We refrain from using the maximum mean discrepancy due to previous reports of sensitivity to hyperparameter settings \citep[see e.g.][]{lueckmann2021benchmarking}.}, and in Figure \ref{fig:ou_mds} the distances between the means of the estimated and approximate ground truth posteriors, for each ratio estimation method.

From this, we observe that \gls{gru-resnet} (orange, top right of Figure \ref{fig:ou_posteriors}) failed to learn both informative summary statistics and an accurate ratio estimator with a low simulation budget, despite the simplicity of the model. In contrast, an identical residual network used for \gls{bresnet} (green, bottom left of Figure \ref{fig:ou_posteriors}) was able to learn a good estimate of the density ratio, even from such a limited simulation budget and with a summary statistic vector of identical size, but with the key difference that the summary statistics were predefined and designed to be informative of the parameter values being inferred. 

This may be seen as an ablation study and suggests that the additional problem of learning summary statistics is the primary contributing factor to the relatively poor performance of \gls{gru-resnet}.

\begin{figure}
\centering
\includegraphics[width=0.44\textwidth, trim=0 0 0 0, clip=True]{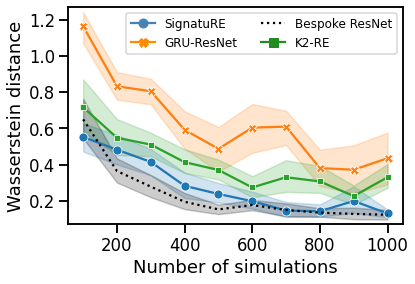}
\caption{(\textbf{Ornstein-Uhlenbeck}) Wasserstein distances (mean + 95\% confidence intervals) between posteriors obtained with each density ratio estimation method and the approximate ground truth posterior.
}\label{fig:ou_swds}
\end{figure}

We also observe that, of the methods that do not use hand-crafted summary statistics, \sigre{} tends to exhibit superior performance. This is apparent from the posterior plots in Figure \ref{fig:ou_posteriors}, and from Figure \ref{fig:ou_swds} in which \sigre{} consistently generates smaller \gls{wd}s than \gls{gru-resnet} and \gls{k2re} and lags only slightly behind \gls{bresnet}. 

From Figure \ref{fig:ou_mds} we see that \sigre{} tends to generate a significantly better parameter point estimate than \gls{gru-resnet} and is additionally a slight improvement on \gls{k2re} in this respect. 
The latter indicates that the success of \sigre{} in the low-simulation-budget regime is not only attributable to the expressive, preexisting feature representations available with \emph{general} kernel methods, but also to the fact that the \emph{sequentialisation} of the kernel employed in \sigre{} captures important information on the time-dependence of the data whereas in \gls{k2re} the data is treated as \emph{iid}.

\begin{figure}
\vspace{.05in}
\centering
\includegraphics[width=0.45\textwidth]{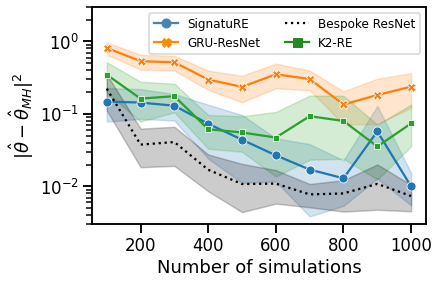}
\caption{(\textbf{Ornstein-Uhlenbeck}) Euclidean distances (mean + 95\% confidence intervals) between posterior means obtained with each density ratio estimation method and the approximate ground truth posterior.
}\label{fig:ou_mds}
\end{figure}

\subsection{Moving average model}

We next consider a simple \gls{ma2}, for which the data-generating process given parameters $\bth = \left(\theta_1, \theta_2\right)$ is
\begin{equation}
    \bx_{t} = \epsilon_{t} + \theta_1 \epsilon_{t-1} + \theta_2 \epsilon_{t-2},\ \ \epsilon_{t} \sim \mathcal{N}\left(0, 1\right).
\end{equation}
We generate $\bx^{*} \sim p\left(\mathbf{x} \mid \bth^{*}\right)$ with $\bth^{*} = \left(0.6, 0.2\right)$ and consider the task of estimating $p\left(\bth \mid \bx^{*}\right)$ given a uniform prior over the triangle given by $\theta_1 + \theta_2 > -1$, $\theta_1 - \theta_2 < 1$, and $\theta_2 < 1$. Such a prior ensures that the model parameters are identifiable \citep{Marin2012}. Here, $\bx$ and $\bx^{*}$ are taken to be of length 50.

As $\tilde{\bs}(\bx)$ we use the variance of the observed stream and the autocorrelations for lags 1 and 2. These give estimates of
\begin{align*}
    \text{Var}\left(\mathbf{X}\right) = 1\, +&\ \theta_1^2 + \theta_2^2,\ \ \
    \rho_{1} = \frac{\theta_1 + \theta_1 \theta_2}{1 + \theta_1^2 + \theta_2^2},\\
    \text{and   }& \rho_2 = \frac{\theta_2}{1 + \theta_1^2 + \theta_2^2},
\end{align*}
respectively, and are thus informative about $\bth$. We once again apply a single linear layer of size 3 following the GRU in \gls{gru-resnet} to match the dimensions of the summary statistics in \gls{bresnet}.

We show in Figure \ref{fig:ma2_swds} the \gls{wd}s between samples from the posteriors estimated with each density ratio estimation method and the approximate ground truth posterior obtained with Metropolis-Hastings \gls{mcmc}. In Figure \ref{fig:ma2_mds}, we show the Euclidean distances between the means of the posteriors estimated with the different density ratio estimators and the approximate ground truth posterior. In this experiment, we once more see that \gls{bresnet} significantly outperforms \gls{gru-resnet} in estimating the shape of the posterior distribution, despite the fact that they use identical residual networks to perform the density ratio estimation and that $\text{dim}\left(\tilde{\bs}\right) = \text{dim}\left(\bs_{\varphi}\right)$. This again suggests that the complex task of learning summary statistics in addition to learning the density ratio is the source of the difference in their performance.

We further observe that \sigre{} outperforms \gls{gru-resnet} both in terms of the \gls{wd} and distances between the estimated and approximate ground truth posterior means for simulation budgets of less than 500. For simulation budgets of 600-1000, \sigre{} and \gls{gru-resnet} display comparable performance according to the \gls{wd}s, while \sigre{} continues to obtain superior posterior mean estimates. Interestingly, \sigre{} additionally yields better estimates of the posterior mean than \gls{bresnet}, despite the fact that this density estimator has a considerable advantage through the use of hand-crafted summary statistics that are known to be informative of the parameters being inferred. 
As in the previous experiment, the success of \sigre{} appears to be attributable not only to the general properties of kernel methods that make them appealing in low-sample regimes -- their ready-made, expressive feature spaces -- but also to the fact that the signature accounts for the ordering of observations. We believe this explains the gap in performance between \gls{k2re} and \sigre{} despite the former also being a kernel method.

\begin{figure}
\vspace{.1in}
\centering
\includegraphics[width=0.43\textwidth, trim=0 0 0 0, clip=True]{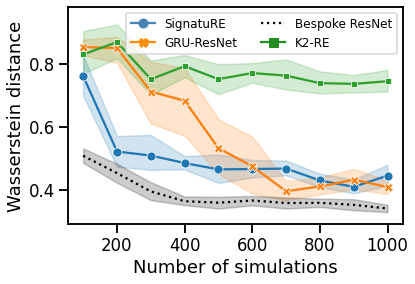}
\vspace{.04in}
\caption{(\textbf{MA(2)}) Wasserstein distances (mean + 95\% confidence intervals) between posteriors obtained with each density ratio estimation method and the approximate ground truth posterior.
}\label{fig:ma2_swds}
\end{figure}

\subsection{Complex, intractable example: partially-observed stochastic epidemic}

Finally, we consider a more complex example with an intractable posterior distribution. The model we consider here is a \gls{gse} model \citep{Kypraios2007EfficientBI}, which simulates the spread of an infection through a fixed population of $N$ individuals. Individuals in the system are initially \emph{susceptible}, can become \emph{infected}, and subsequently enter a \emph{recovered} state in which they are no longer susceptible to reinfection. In a time interval $\delta t$, infections, recoveries, and an absence of activity occur with probabilities
\begin{align*}
    &{} P_{I} := P\left[(\delta X_{t}, \delta Y_{t}) = (-1, 1) \mid \sigma_t \right] = \beta X_t Y_t \delta t + o(\delta t),\\
    &{} P_{R} := P\left[(\delta X_{t}, \delta Y_t) = (0, -1) \mid \sigma_t \right] = \gamma Y_t \delta t + o(\delta t),\\
    &{} P\left[(\delta X_{t}, \delta Y_t) = (0, 0) \mid \sigma_t \right] = 1 - (P_I + P_R),
\end{align*}
respectively, where $X_t$ and $Y_t$ are the number of susceptible and infected agents at time $t \in [0, T]$, respectively, $\sigma_t$ is a sigma-algebra generated by the process up until time $t$, and $\bth = \left(\beta, \gamma\right)$ is the model parameter.

We simulate the model using the Gillespie algorithm \citep{Gillespie} and observe the series $\bx = \left(X_{i\Delta t}, Y_{i\Delta t}\right)_{i = 0}^{D} \in \mathcal{S}_{D+1}\left(\mathbb{R}^2\right)$ at regular time intervals of length $\Delta t = 0.5$ with $D=100$.
We consider the task of estimating $p\left( \bth \mid \bx^*\right)$ for $\bx^{*} \sim p\left(\bx \mid \bth^{*}\right)$, $\bth^{*} = \left(10^{-2}, 10^{-1}\right)$, and priors $\beta \sim \Gamma\left(0.1, 2\right)$ and $\gamma \sim \Gamma\left(0.2, 0.5\right)$. 
To sample from the posterior in this case, we use a \gls{sir} scheme\footnote{Due to the complicated target distribution, the Metropolis--Hastings scheme adopted in the rest of this paper performed poorly.}: we sample $\mathcal{T} = \lbrace{\bth_m \rbrace}_{m=1}^{M}$ from the prior, before resampling $\lbrace{\tilde{\bth}_m\rbrace}_{m=1}^{\tilde{M}}$ from $\mathcal{T}$, where each sample in $\mathcal{T}$ has weight proportional to the density ratio estimated by the classifiers. We take $M = 5\times 10^4$ and $\tilde{M} = 10^3$.

\begin{figure}
\vspace{.09in}
\centering
\includegraphics[width=0.46\textwidth, trim=0 5 0 0, clip=True]{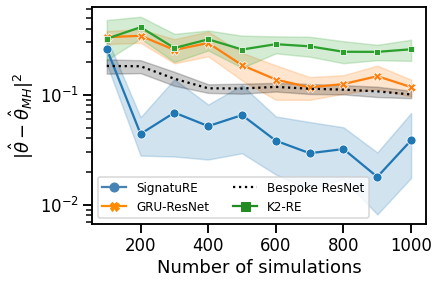}
\vspace{.06in}
\caption{(\textbf{MA(2)}) Euclidean distances (mean + 95\% confidence intervals) between posterior means obtained with each density ratio estimation method and the approximate ground truth posterior mean.
}\label{fig:ma2_mds}
\end{figure}

\begin{table*}[t]
\caption{Median Wasserstein distance from \gls{smcabc} posterior for the partially-observed epidemic model (from 10 seeds). Smaller values are better. {\bf Bold} and {\it italics} indicate best and second-best, respectively, of the methods that do not use pre-defined summary statistics.}\label{tab:wass}
\begin{center}
\begin{tabular}{c c c c c c}
\textbf{Method} & \multicolumn{5}{c}{\textbf{Simulation budget}}\\
 {} & \textbf{50}{} & \textbf{100} & \textbf{200} & \textbf{500} & \textbf{1000} \\
\hline & {}  & {} & {} & {} & {} \\
 \textsc{GRU-ResNet}     & 0.434 & 0.425 & 0.355 & 0.273 & \emph{0.090} \\
 K2-RE          & \emph{0.417} & 0.432 & 0.407 & 0.454 & 0.431 \\
 K2-RE-5        & 0.440 & 0.427 & 0.374 & \emph{0.206} & 0.255 \\
 \textsc{SignatuRE}      & 0.430 & \emph{0.411} & \emph{0.351} & 0.513 & 0.321 \\
 \textsc{SignatuRE-5}    & {\bf 0.241} & {\bf 0.333} & {\bf 0.176} & {\bf 0.133} & {\bf 0.083} \\\hline
 \textsc{Bespoke ResNet}     & 0.379 & 0.222 & 0.146 & 0.104 & 0.092 \\
\end{tabular}
\end{center}
\end{table*}

\glsreset{smcabc}
In this instance, the ground truth posterior distribution is not available for comparison. For this reason, we assess the quality of inferences by comparing against the posterior obtained from \gls{smcabc} \citep{beaumont2009adaptive} in which we use the Euclidean distance between time-series 
\begin{equation}
    \sum_{i=0}^{D} \| \bx_i - \bx^{*}_i \|^2_2
\end{equation}
as the distance measure with $10^7$ simulations, a Gaussian kernel, and $\epsilon$ decay factor equal to 0.8. We again compare \sigre{} with \gls{gru-resnet}, \gls{bresnet}, and \gls{k2re}. For \gls{bresnet}, we use the mean of each series, log variance of each series, autocorrelation coefficients for lags 1 and 2 of each series, and the cross-correlation coefficient between the two series as $\tilde{\bs}(\bx)$, which are common summary statistics for stochastic kinetic models \citep{papamakarios2019sequential, Greenberg2019}. For \gls{gru-resnet}, we apply a single linear layer of size 9 to match the dimensions of $\tilde{\bs}(\bx)$.

We present the median Wasserstein distance between the estimated posteriors and the approximate ground truth posterior from \gls{smcabc} in Table \ref{tab:wass}, in which suffix ``-5'' indicates that $K=5$ for kernel methods (otherwise $K=1$ is used as before). We take $q = B_{\rm min}(K+1)$ components in the Nystr\"{o}m approximation for both \sigre{} and K2-RE, where $B_{\rm min} = 50$ is the minimum simulation budget in this experiment. Median values are obtained by repeating the inference procedure over 10 different random seeds using the same pseudo-observed data. Of the methods that must learn summary statistics (i.e. all but \textsc{Bespoke ResNet}), our methods are either best ({\bf bold}) or second-best (\emph{italics}) for all budgets, but the improvement in performance over \gls{gru-resnet} at a budget of 1000 simulations is minor. While this demonstrates that the range of applicability of \sigre{} may be limited, it nonetheless also demonstrates that \sigre{} can be preferable under extreme restrictions on the simulation budget, as can be the case in many real-world contexts.

\section{DISCUSSION}

This paper discusses the use of signature transforms as automatic and effective feature extractors for likelihood (ratio) estimation. 
Our method, based on universal kernels for sequential data and termed \sigre{}, delivers competitive performance even when sample numbers are very low. 
Indeed, our simulation studies suggest that using signatures as features improves upon a time-series specialised \textsc{GRU-ResNet} or kernels based on maximum mean discrepancies in low-simulation-budget scenarios. We propose that this can be understood in the following way: while \textsc{GRU-ResNet} must learn adequate summary statistics -- which can be difficult for low simulation budgets -- and K2-RE uses a kernel maximum mean discrepancy estimator that treats the points in the time-series as exchangeable, destroying important dependencies, \sigre{} uses expressive ready-made geometric features for paths which take the ordering of points into account. 
In our experiments, \sigre{} was only consistently outperformed by a classifier that used bespoke hand-crafted summary statistics which were constructed by carefully inspecting the model structure. For real, complex simulators, such an approach is infeasible, making the proposed method appealing. 

\subsubsection*{Acknowledgements}
The authors thank Harald Oberhauser and the anonymous reviewers for their helpful feedback. JD is supported by the EPSRC Centre For Doctoral Training in Industrially Focused Mathematical Modelling (EP/L015803/1) in collaboration with Improbable.

\bibliography{references}


\clearpage
\appendix

\thispagestyle{empty}

\onecolumn \makesupplementtitle

\section{EXPERIMENT DETAILS}

\subsection{Further details on K2-RE}

To test the hypothesis that the signature kernel is responsible for the improved performance seen in the experiments presented in the main text, we construct and compare an alternative kernel-based classifier to compare against. The design of this classifier is chosen to match exactly that of \sigre{}, with an important change: the kernel $k$ is no longer taken to be the signature kernel, but instead a kernel based on the K2-ABC \citep{Park2016}:
\begin{equation}
    k\left(\bx, \tilde{\bx}\right) = \exp\left(-\frac{\widehat{\textsc{MMD}}^2\left(\mu_{\bx}, \mu_{\tilde{\bx}}\right)}{\epsilon}\right),
\end{equation}
where
\begin{multline}
    \widehat{\textsc{MMD}}^2\left(\mu_{\bx}, \mu_{\tilde{\bx}}\right) = -\frac{2}{n_{\bx} n_{\tilde{\bx}}} \sum_{i=1}^{n_{\bx}} \sum_{j=1}^{n_{\tilde{\bx}}} \chi\left(\bx_{i}, \tilde{\bx}_{j}\right)
    + \frac{1}{n_{\bx}\left(n_{\bx} - 1\right)} \sum_{\substack{i=1}}^{n_{\bx}} \sum_{\substack{j\neq i}} \chi\left(\bx_{i}, \bx_{j}\right)\\
    + \frac{1}{n_{\tilde{\bx}}\left(n_{\tilde{\bx}} - 1\right)} \sum_{\substack{i=1}}^{n_{\tilde{\bx}}} \sum_{\substack{j\neq i}} \chi\left(\tilde{\bx}_{i}, \tilde{\bx}_{j}\right)
\end{multline}
is an unbiased estimate of the kernel maximum mean discrepancy between measures $\mu_{\bx}$ and $\mu_{\tilde{\bx}}$ for an appropriate kernel $\chi$ \citep[see Section 3.][]{Park2016}. We use a Gaussian RBF and the median heuristic \citep[Section 4.][]{Park2016} for $\chi$. 

Comparing against an alternative kernel classifier that does not account for the ordering of the points $\bx_i$ in $\bx$ allows us to test the hypothesis that it is specifically the \emph{signature} kernel, and not just kernel methods in general, that allow us to achieve the improved performance at low simulation budgets.

\subsection{Tuning kernel parameters}

To optimise the kernel parameters for \textsc{SignatuRE} and K2-RE, we use 5-fold cross-validation and Bayesian optimisation via a tree Parzen estimator with the following priors:
\begin{enumerate}
    \item a log-uniform prior with bounds $\left[\log{10^{-3}}, \log{10^{3}}\right]$ for all lengthscale parameters;
    \item a log-uniform prior with bounds $\left[\log{10^{-5}}, \log{10^{4}}\right]$ for the regularisation parameters.
\end{enumerate}
For this purpose, we make use of the \texttt{hyperopt} python package (Bergstra et al., 2013).

\subsection{Training the ResNet models}

For both \textsc{GRU-ResNet} and B\textsc{espoke} \textsc{ResNet}, the \textsc{ResNet} consists of two hidden layers of 50 units with ReLU activations, which has previously been seen to produce state-of-the-art performance in likelihood-free density ratio estimation tasks \citep{Durkan2020, Lueckmann2017}. We follow \citet{Durkan2020} and use Adam \citep{kingma2014adam} to train the network weights, along with a training batch size of 50 and learning rate of $5\times 10^{-4}$. We furthermore reserve 10\% of the data for validation, and stop training when the validation error does not improve over 20 epochs to avoid overfitting. For these density ratio estimators, we use the \texttt{sbi} python package \citep{sbi}.

\subsection{Sampling with Metropolis-Hastings}

Unless stated otherwise, we obtain samples from both the approximate ground-truth posteriors and the posterior distributions estimated with density ratios with Metropolis-Hastings Markov chain Monte Carlo. We use a normal proposal distribution $q(\tilde{\bth} \mid \boldsymbol{\theta}) = \mathcal{N}\left(\boldsymbol{\theta}, \ell^2\Sigma\right)$ with covariance matrix $\Sigma$, which estimate by performing a trial run of 50,000 steps with a diagonal proposal covariance matrix \citep[see e.g. the guidelines in][Section 12.2]{gelman2013bayesian} and setting $\ell = 2/\sqrt{d}$ for $\boldsymbol{\theta}\in \mathbb{R}^d$ \citep{roberts1997weak}. This works well if the posterior is approximately normal \citep[see][]{schmon2021optimal}. Once $\Sigma$ is estimated, we run one further chain for 100,000 steps, and thin by retaining every 100th sample. We furthermore start every chain from the true parameter values $\boldsymbol{\theta}^{*}$.

\subsection{Confidence interval evaluations}

In Figures 3-6 in the main text, the 95\% confidence intervals are bootstrap confidence intervals obtained by running the training procedures at different seeds and subsequently applying the trained ratio estimators to the task of obtaining the posterior for the same pseudo-observed data in each case.

\vfill

\end{document}